\documentclass[sigconf]{acmart}

\usepackage{algorithm}
\usepackage{algorithmic}
\usepackage{verbatim}
\usepackage{multirow}
\usepackage{enumitem}
\usepackage{amsmath, amsthm}
\usepackage{makecell}

\usepackage{rotating}  
\usepackage{graphicx}

\newtheorem{definition}{Definition}[section]
\newtheorem{theorem}{Theorem}[section]

\AtBeginDocument{%
  \providecommand\BibTeX{{%
    \normalfont B\kern-0.5em{\scshape i\kern-0.25em b}\kern-0.8em\TeX}}}

\copyrightyear{2024} 
\acmYear{2024} 
\setcopyright{acmlicensed}
\acmConference[WWW '24]{Proceedings of the ACM Web Conference 2024}{May 13--17, 2024}{Singapore, Singapore}
\acmBooktitle{Proceedings of the ACM Web Conference 2024 (WWW '24), May 13--17, 2024, Singapore, Singapore}
\acmISBN{979-8-4007-0171-9/24/05}
\acmDOI{10.1145/3589334.3645693}
\begin{document}

\title{Federated Heterogeneous Graph Neural Network for Privacy-preserving Recommendation}

\author{Bo Yan}
\orcid{0000-0003-1750-3803}
\affiliation{%
  \institution{Beijing University of Posts and Telecommunications}
  \city{Beijing}
  \country{China}
}
\email{boyan@bupt.edu.cn}

\author{Yang Cao}
\orcid{0000-0002-6424-8633}
\affiliation{%
  \institution{Hokkaido University}
  \city{Sapporo}
  \country{Japan}
}
\email{yang@ist.hokudai.ac.jp}

\author{Haoyu Wang}
\orcid{0009-0004-1010-0723}
\affiliation{%
  \institution{Beijing University of Posts and Telecommunications}
    \city{Beijing}
  \country{China}
}
\email{wanghaoyu666@bupt.edu.cn}

\author{Wenchuan Yang}
\orcid{0000-0003-3194-1690}
\affiliation{%
 \institution{National University of Defense Technology}
 \city{Changsha}
 \country{China}}
 \email{wenchuanyang97@163.com}

\author{Junping Du}
\orcid{0000-0002-9402-3806}
\affiliation{%
  \institution{Beijing University of Posts and Telecommunications}
  \city{Beijing}
  \country{China}}
\email{junpingdu@126.com}

\author{Chuan Shi}
\orcid{0000-0002-3734-0266}
\authornote{Corresponding author.}
\affiliation{%
  \institution{Beijing University of Posts and Telecommunications}
  \city{Beijing}
  \country{China}}
\email{shichuan@bupt.edu.cn}

\renewcommand{\shortauthors}{Bo Yan et al.}

\begin{abstract}
The heterogeneous information network (HIN), which contains rich semantics depicted by meta-paths, has emerged as a potent tool for mitigating data sparsity in recommender systems. Existing HIN-based recommender systems operate under the assumption of centralized storage and model training. However, real-world data is often distributed due to privacy concerns, leading to the semantic broken issue within HINs and consequent failures in centralized HIN-based recommendations. In this paper, we suggest the HIN is partitioned into private HINs stored on the client side and shared HINs on the server. Following this setting, we propose a federated heterogeneous graph neural network (FedHGNN) based framework, which facilitates collaborative training of a recommendation model using distributed HINs while protecting user privacy. Specifically, we first formalize the privacy definition for HIN-based federated recommendation (FedRec) in the light of differential privacy, with the goal of protecting user-item interactions within private HIN as well as users' high-order patterns from shared HINs. To recover the broken meta-path based semantics and ensure proposed privacy measures, we elaborately design a semantic-preserving user interactions publishing method, which locally perturbs user's high-order patterns and related user-item interactions for publishing. Subsequently, we introduce an HGNN model for recommendation, which conducts node- and semantic-level aggregations to capture recovered semantics. Extensive experiments on four datasets demonstrate that our model outperforms existing methods by a substantial margin (up to 34\% in HR@10 and 42\% in NDCG@10) under a reasonable privacy budget (e.g., $\epsilon=1$).
\end{abstract}





\maketitle

\section{Introduction}

\begin{figure}[t]
    \centering
    \includegraphics[scale=.53]{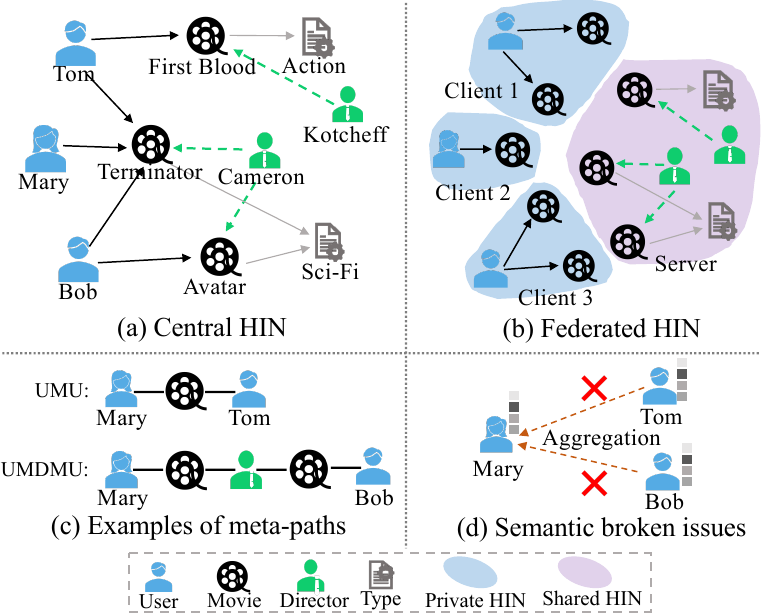}
    \caption{Comparison of a HIN in the centralized setting and federated setting}
    \label{fig:hin}
\end{figure}


Recommender systems play a crucial role in mitigating the challenges posed by information overload in various online applications \cite{DBLP:conf/www/ZhouZY23, DBLP:journals/dase/ChenCWFZLDXG23, DBLP:journals/dase/LiCGW23}. However, their effectiveness is limited by the sparsity of user interactions \cite{DBLP:conf/sigir/LiuOMM20, DBLP:conf/kdd/ZhengMGZ21, DBLP:conf/kdd/Lu0S20}. To tackle this issue, heterogeneous information networks (HIN), containing multi-typed entities and relations, have been extensively utilized to enhance the connections of users and items \cite{DBLP:conf/kdd/ZhaoYLSL17, DBLP:conf/esws/SchlichtkrullKB18, DBLP:journals/tkde/ShiHZY19, DBLP:conf/www/HuDWS20}. As a fundamental analysis tool in HIN, \textit{meta-path} \cite{DBLP:journals/tkde/ShiLZSY17}, a relation sequence connecting node pairs, is widely used to capture rich semantics of HINs. Different meta-path can depict distinct semantics, as illustrated in Figure \ref{fig:hin}, the meta-path \textit{UMU} in the HIN for movie recommendation signifies that two users have watched the same movie, and the \textit{UMDMU} depicts that two users have watched movies directed by the same director. Most HIN-based recommender methods leverage meta-path based semantics to learn effective user and item embeddings \cite{DBLP:journals/tkde/ShiHZY19, DBLP:conf/kdd/HuSZY18}. Among them, early works integrate meta-path based semantics into user-item interaction modeling to enhance their representations \cite{DBLP:journals/tkde/ShiHZY19, DBLP:conf/kdd/ZhaoYLSL17}. In recent years, the advent of graph neural networks (GNNs) have introduced a powerful approach to automatically capture meta-path based semantics, achieving remarkable results \cite{DBLP:conf/kdd/WangLHS21, DBLP:conf/www/0004ZMK20, DBLP:conf/sigir/ZhengMW22}. These GNN-based methods aggregate node embeddings along meta-paths to fuse different semantics, known as meth-path based neighbor aggregation \cite{DBLP:conf/kdd/FanZHSHML19,DBLP:conf/aaai/JiZWSWTLH21, DBLP:conf/www/WangJSWYCY19, DBLP:conf/kdd/ZhengMGZ21}, providing a more flexible framework for HIN-based recommendations.

Existing HIN-based recommendations hold a basic assumption that the data is centrally stored. As shown in Figure \ref{fig:hin}(a) and (c), under this assumption, the entire HIN is visible and can be directly utilized to capture meta-path based semantics for recommendation. However, this assumption may not hold in reality since the user-item interaction data is highly privacy-sensitive, and the centralized storage can leak user privacy. Additionally, strict privacy protection regulations, such as the General Data Protection Regulation (GDPR)\footnote{https://gdpr-info.eu}, prohibit commercial companies from collecting and exchanging user data without the user's permission. Therefore, centralized data storage may not be feasible in the future. As a more realistic learning paradigm, federated learning (FL) \cite{DBLP:conf/aistats/McMahanMRHA17, DBLP:journals/tist/YangLCT19} has emerged to enable users to keep their personal data locally and jointly train a global model by transmitting only intermediate parameters. Federated recommendation (FedRec) is a crucial application of FL in the recommender scenario and many works have been dedicated to FedRec in recent years\cite{DBLP:journals/expert/ChaiWCY21, DBLP:journals/corr/abs-1901-09888, DBLP:conf/recsys/LinP021, DBLP:journals/corr/abs-2008-07759}. Most of them focus on traditional matrix factorization (MF) based FedRec \cite{DBLP:journals/expert/ChaiWCY21, DBLP:conf/recsys/LinP021}, where user embeddings are kept locally updated, and gradients of item embeddings are uploaded to the server for aggregation. Recently, a few studies have explored GNN-based FedRec \cite{wu2022federated, DBLP:journals/tist/LiuYFPY22, DBLP:conf/cikm/LuoXS22}. They train local GNN models on the user-item bipartite graph and transmit gradients of embeddings and model parameters. Despite their success, they still suffer from data sparsity issues, which are further compounded by the distributed data storage.



A natural solution is utilizing HINs to enrich the sparse interactions. However, developing HIN-based FedRec is non-trivial. It faces two significant challenges. 1) \textit{How to give a formal definition of privacy in HIN-based FedRec?}
Traditional definitions based on solely user-item interactions may be infeasible in the HIN-based FedRec. Besides private user-item interactions, HIN-based FedRec can also utilize additional shared knowledge that contains no privacy and can be shared among users (e.g., movie-type and movie-director relations in Figure \ref{fig:hin}(a)), which may also expose the user’s high-order patterns, such as their favorite types of movies. Therefore, we should carefully consider the realistic privacy constraints of HIN-based FedRec and give a rigorous privacy definition so that the privacy can be rigorously protected.  2) \textit{How to recover the broken semantics for HIN-based FedRec while protecting the defined privacy?} The HIN is stored in a distributed manner in FedRec, as shown in Figure \ref{fig:hin}(b), and users can only access their one-hop neighbors (interacted items). As a result, the integral semantics depicted by the meta-path are broken, leading to failure to conduct meta-path based neighbor aggregations, which is the key component for HIN-based recommendation. As depicted in Figure \ref{fig:hin}(c) and (d), the meta-path based neighbor aggregations fail because of the broken semantics UMU and UMDMU. However, due to privacy constraints, it’s unrealistic to directly exchange the user interaction data. Thus, it’s challenging to recover the semantics with privacy guarantees.

To tackle these challenges, we delve into the study of HIN-based FedRec and propose a \underline{Fed}erated \underline{H}eterogeneous \underline{G}raph \underline{N}eural \underline{N}etwork (FedHGNN) for privacy-preserving recommendations. 1) To clarify the privacy that should be protected, we present a formal privacy definition for HIN-based FedRec. We suggest a realistic setting that the entire HIN is partitioned into private HINs stored on the client side and shared HINs on the server side. Under this setting, we rigorously formalize two kinds of privacy for HIN-based FedRec in the light of differential privacy \cite{DBLP:journals/fttcs/DworkR14}, including the privacy reflecting the user's high-order patterns from shared HINs and the privacy of user-item interactions with specific patterns in the private HIN. 2) To recover the broken semantics while protecting proposed privacy, we introduce a semantic-preserving user interaction publishing algorithm, the core of which is a two-stage perturbation mechanism. Specifically, the first stage perturbs the user's high-order patterns from shared HINs by a dedicated designed exponential mechanism (EM) \cite{DBLP:journals/fttcs/DworkR14}. To maintain the utility of perturbed data, we select the shared HINs that are relevant to the user's true patterns with a higher probability. The second stage perturbs user-item interactions within each selected shared HIN in a degree-preserving manner, which avoids introducing more noise and also enhances the interaction diversity. Users perturb their interactions locally by the two-stage perturbation mechanism and upload them to the server for recovering semantics. We also provide rigorous privacy guarantees for this publishing process. Based on the recovered semantics, we further propose a general heterogeneous GNN model for recommendation, which captures semantics through a two-level meta-path-guided aggregation.

The major contributions of this paper are summarized as follows:

\begin{itemize}
  \item To the best of our knowledge, this is the first work to study the HIN-based FedRec, which is an important and practical task in real-world scenarios. 
  
  \item We design a FedHGNN framework for HIN-based FedRec. We give a formal privacy definition and propose a novel semantic-preserving perturbation method to publish user interactions for recommendation. We also give rigorous privacy guarantees for the publishing process.
  
  \item We conduct extensive experiments on four real-world datasets, which demonstrate that FedHGNN improves existing FedRec methods by up to 34\% in HR@10 and 42\% in NDCG@10, while maintaining a reasonable privacy budget. Additionally, FedHGNN achieves comparable or even superior results compared to centralized methods.

\end{itemize}

\section{Preliminary}
In this paper, we conduct HIN-based FedRec for implicit feedback. Let $U$ and $I$ denote the user set and item set. We give the related concepts as follows. 
\subsection{Heterogeneous Information Network}
\begin{definition}
\textbf{Heterogeneous Information Network (HIN)} \cite{DBLP:journals/tbd/WangBSFYY23}. A HIN $G=(V, E)$ consists of an object set $V$ and a link set $E$. It is also associated with an object type mapping function $\phi: V \rightarrow \mathcal{A}$ and a link type mapping function $\psi: E \rightarrow \mathcal{R}$$     $. $\mathcal{A}$ and $\mathcal{R}$ are the predefined sets of object and link types, where $|\mathcal{A}|+|\mathcal{R}| > 2$. 
\end{definition}


\begin{definition}
\textbf{Meta-path}. Given a HIN $G$ with object types $\mathcal{A}$ and link types $\mathcal{R}$, a meta-path $\rho$ can be denoted as a path in the form of $A_1 \xrightarrow{R_1} A_2 \xrightarrow{R_2}\cdots \xrightarrow{R_l} A_{l+1}$, where $A_i \in \mathcal{A}$ and $R_i \in \mathcal{R}$. Meta-path describes a composite relation $R=R_1 \circ R_2 \circ ... \circ R_l$ between object $A_1$ and $A_{l+1}$, where $\circ$ denotes the composition operator on relations. Then given a node $v$ and a meta-path $\rho$, the \textbf{meta-path based neighbors} $\mathcal{N}_v^{\rho}$ of $v$ are the nodes connecting with $v$ via the meta-path $\rho$. In a HIN, the rich semantics between two objects can be captured by multiple meta-paths.
\end{definition}


\subsection{Privacy Definition}
\label{sec:3.2}

\begin{definition}
    \textbf{Private HIN}. A private HIN $G_p = (V_p, E_p)$ is defined as a subgraph of $G$. It is associated with an object type mapping function $\phi_p:V_p \rightarrow \mathcal{A}$ and a link type mapping function $\psi_p:E_p \rightarrow \mathcal{R}_p$, where $\mathcal{R}_p \subset \mathcal{R}$ is the set of private link types. A \textbf{user-level private HIN} contains a user $u \in V_p$ and its interacted item set $I^u \subset I$. The link set $E^u_p$ exists between $u$ and $I^u$ denoting personally private interactions.
\end{definition}



\begin{definition}
    \textbf{Shared HIN}. A shared HIN $G_s=(V_s,E_s)$ is defined as a subgraph of  $G$. It is associated with an object type mapping function $\phi_s: V_s \rightarrow \mathcal{A}$ and a link type mapping function $\psi_s: E_s \rightarrow \mathcal{R}_s$, where $\mathcal{R}_s$ is the set of shared link types. 
\end{definition}

As depicted in Figure \ref{fig:hin}(a) and (b), under federated setting, the movie network is divided into user-level private HINs stored in each client and shared HINs stored in the server. A user-level private HIN includes a user's private interactions while shared HINs contain shared knowledge such as movie-director relations. 


A user $u$ could associate with many shared HINs based on interacted items. For example, Figure \ref{fig:hin} (a) and (b) depict that two shared HINs are related to Tom, and one shared HIN is related to Mary. These user-related shared HINs reflect high-order patterns of users (e.g., favorite types of movies) and should be protected. We call this privacy as \textit{semantic privacy}, denoted as a user-related shared HIN list $g = (g_1, \cdots, g_{|\mathcal{G}_s|}) \in \{0,1\}^{|\mathcal{G}_s|}$, where $\mathcal{G}_s$ denotes the whole shared HIN set. Then we formalize semantic privacy as follows:  



\begin{definition}
    \textbf{$\epsilon$-Semantic Privacy}. Given a user-related shared HIN list $g$, a perturbation mechanism $\mathcal{M}$ satisfies $\epsilon$-semantic privacy if and only if for any $\hat{g}$, such that $g$ and $\hat{g}$ only differ in one bit, and any $\tilde{g} \in range(\mathcal{M})$, we have $\frac{Pr[\mathcal{M}(g)=\tilde{g}]}{Pr[\mathcal{M}(\hat{g})=\tilde{g}]} \leq e^{\epsilon}$. 
\end{definition}

Besides, each user also owns a adjacency list $a=(a_1,\cdots, a_{|I|}) \in {\{0,1\}}^{|I|}$. Given $g$, we can extract a subset $I_s$ from the whole item set $I$, which is called \textit{semantic guided item set}. Similarly, we can obtain \textit{semantic guided adjacency list} denoted as $a_s=(a_{s1},\cdots, a_{s|I_s|}) \in {\{0,1\}}^{|I_s|}$, which depicts the user-item interactions with specific patterns and should also be protected. we call this privacy as \textit{semantic guided interaction privacy} and formalize as:

\begin{definition}
    \textbf{$\epsilon$-Semantic Guided Interaction Privacy}. Given a semantic guided adjacency list $a_s$, a perturbation mechanism $\mathcal{M}$ satisfies $\epsilon$-semantic guided interaction privacy if and only if for any $\hat{a_s}$, such that $a_s$ and $\hat{a_s}$ only differ in one bit, and any $\tilde{a_s} \in range(\mathcal{M})$, we have $\frac{Pr[\mathcal{M}(a_s)=\tilde{a_s}]}{Pr[\mathcal{M}(\hat{a_s})=\tilde{a_s}]} \leq e^{\epsilon}$.
\end{definition}

$\epsilon$ is called the privacy budget that controls the strength of privacy protection. It is obvious that if a perturbation algorithm satisfies these definitions, the attacker is hard to distinguish the user's high-order pattern as well as the true interacted items. 




\begin{figure*}[h]
    \centering
    \includegraphics[scale=.6]{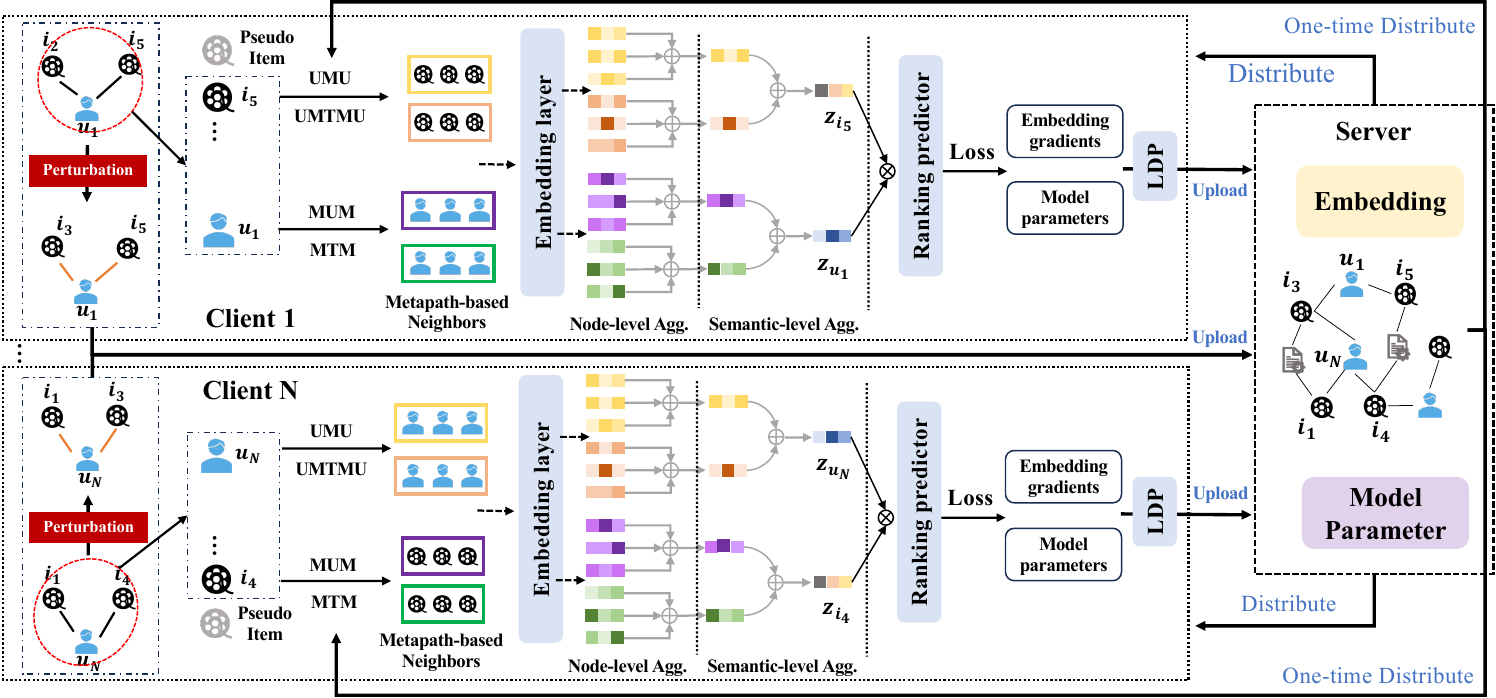}
    \caption{The overall framework of FedHGNN}
    \label{fig:overall}
\end{figure*}

\subsection{Task Formulation}
Based on the above preliminaries, we define our task as follows:
\begin{definition}
    \textbf{HIN-based FedRec}. Given user-level private HINs $\mathcal{G}_p = \{G_p^{u_1}, G_p^{u_2},..., G_p^{u_{|U|}}\}$ and shared HINs $\mathcal{G}_s= \{G_s^1, G_s^2,...,G_s^m\}$. The $G_p^{u_i}$ corresponding to the user $u_i \in U$ is stored in the $i$-th client, while $\mathcal{G}_s$ is stored in the server. We aim to collaboratively train a global model based on these distributed HINs with satisfying $\epsilon$-semantic privacy and $\epsilon$-semantic guided interaction privacy, which can recommend a ranked list of interested items for each user $u \in U$.
\end{definition}

\section{Methodology}
In this section, we give a detailed introduction to the proposed model FedHGNN. We first give an overview of FedHGNN. Then we present two main modules of FedHGNN, the semantic-preserving user-item interaction publishing and heterogeneous graph neural networks (HGNN) for recommendation. Finally, we give a privacy analysis of the proposed publishing process. 

\subsection{Overview of FedHGNN}
Different from existing FedRec systems that only utilize user-item interactions, FedHGNN also incorporates HINs into user and item modeling, which can largely alleviate the cold-start issue caused by data sparsity. Besides, as a core component of FedHGNN, the semantic-preserving user-item publishing mechanism recovers semantics with rigorous privacy guarantees, which can be applied to all meta-path based FedRec systems technically. We present the overall framework of FedHGNN in Figure \ref{fig:overall}. As can be seen, it mainly includes two steps, i.e., user-item interaction publishing and HGNN-based federated training. At the user-item interaction publishing step, each client perturbs local interactions using our two-stage perturbation mechanism and then uploads the perturbed results to the server. After the server receives local interactions from all clients, it can form an integral perturbed HIN, which is then distributed to each client to recover the meta-path based semantics. Note that the publishing step is only conducted once in the whole federated training process. At the federated training step, clients collaboratively train a global recommendation model based on recovered neighbors, which performs node-level neighbor aggregations followed by semantic-level aggregations. Then a ranking loss is adapted to optimize embedding and model parameters. At each communication round, each participating client locally trains the model and uploads the embedding and model gradients to the server for aggregations. To further protect privacy when uploading gradients, we apply local differential privacy (LDP) to the uploaded gradients. Besides, following previous work \cite{DBLP:journals/tist/LiuYFPY22,wu2022federated}, we also utilize pseudo interacted items during local training. 


\subsection{Semantic-preserving User-item Interactions Publishing}
To recover the semantics of the centralized HIN (obtaining the meta-path based neighbors), directly uploading the adjacency list $a^u$ to the server can not satisfy the privacy definition because the user-item interactions are exposed. To address this, we first present a naive solution based on random response (RR) \cite{DBLP:journals/fttcs/DworkR14} and illustrate its defects of direct applications to our task. Then we give detailed introductions of our proposed two-stage perturbation mechanism for user-item interaction publishing. As depicted in Figure \ref{fig:perturbation}, it first perturbs the user-related shared HINs and then perturbs the user-item interactions within selected shared HINs, which not only achieves semantic-preserving but also satisfies the defined privacy.

\noindent \textbf{Random response (RR)}. As many homogeneous graph metrics publishing \cite{DBLP:conf/ccs/QinYYKX017, DBLP:journals/corr/abs-2202-10209, DBLP:journals/tkde/YeHAMX22, DBLP:conf/aaai/WangCWP0Y17}, a straw-man approach is directly utilizing RR \cite{DBLP:journals/fttcs/DworkR14} to perturb each user's adjacency list $a^u$, i.e., the user flips each bit of $a^u$ with probability $p = \frac{1}{1+e^{\epsilon}}$. However, this naive strategy faces both privacy and utility limitations. For privacy, although it satisfies the $\epsilon$-semantic guided interaction privacy, it can not achieve our $\epsilon$-semantic privacy goal. As for utility, it has been theoretically proved that RR would make a graph denser \cite{DBLP:conf/ccs/QinYYKX017}. Unfortunately, there exists perturbation enlargement phenomenon \cite{DBLP:conf/aaai/ZhangWZSZZ22} in the HGNNs, i.e., introducing more edges may harm the HGNN's performance, which is also confirmed in our latter experiments. Besides, RR fails to accommodate the semantic-preserving since it perturbs all bits of $a^u$. We can only perturb the semantic guided item set to preserve semantics but this will expose the user's high-order patterns. Furthermore, the denser graph largely hinders the training speed and compounds the communication overhead in the federated setting. 

\begin{figure}[t]
    \centering
    \includegraphics[scale=.55]{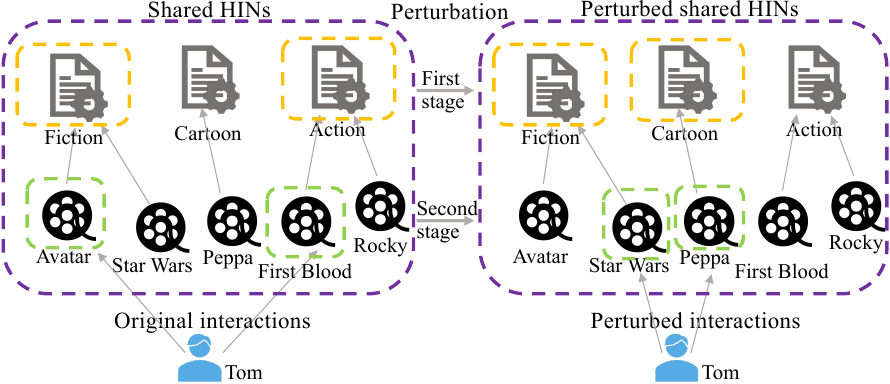}
    \caption{The two-stage perturbation mechanism for user-item interaction publishing}
    \label{fig:perturbation}
\end{figure}


\noindent \textbf{User-related shared HIN perturbation}. Based on the above analysis, we propose a two-stage perturbation mechanism. The first stage performs user-related shared HIN perturbation, which utilizes EM to select shared HINs for publishing. Intuitively, the true user-related shared HIN should be selected with a high probability. Therefore, according to the theory of EM, for a user $u$ with related shared HIN set $g$, we design the utility of selecting a shared HIN as follows:

\begin{equation}
\begin{aligned}
q(g, u, G_s) 
&= sim(G_s, \mathcal{G}_s^u)\\
&= \max_{G^{'}_s \in \mathcal{G}_s^u} \{\frac{1}{2}(cos(e_{G_s}, e_{G^{'}_s})+1\},
\end{aligned}
\label{eq:quality}
\end{equation}
where $\mathcal{G}_s^u$ is the shared HIN set of $u$ and $G_s \in \mathcal{G}_s^u$ is the selected shared HIN. $e_{G_s}$ is the representation of $G_s$ which is the average of related items' embeddings. Eq. (\ref{eq:quality}) indicates that if a shared HIN $G_s$ is more similar to user-related shared HIN set $\mathcal{G}_s^u$, it should be selected with a high probability. In this regard, the similarity function has multiple choices. We choose the highest cosine similarity score among $\mathcal{G}_s^u$ as the similarity function mainly in consideration of achieving a smaller sensitivity to obtain higher utility. In this way, the sensitivity $\Delta q$ is:
\begin{equation}
    \Delta q = \max\limits_{G_s} \max \limits_{g \sim \hat{g}} |q(g,u, G_s)-q(\hat{g},u, G_s)| = 1, 
\end{equation}
where $g \sim \hat{g}$ denotes that $g$ and $\hat{g}$ only differ in one bit. Then according to the EM, a shared HIN $G_s$ is selected with probability:

\begin{equation}
    \text{Pr}(G_s) = \frac{\exp(\epsilon q(g,u,G_s)/(2 \Delta q))}{\sum_{G_s^{\prime} \subset \mathcal{G}_s} \exp(\epsilon q(g,u,G_s^{\prime})/(2 \Delta q))}.
\end{equation}
The above selection process is repeated $|\mathcal{G}^u_s|$ times without replacement to ensure diversity. Then we can obtain the perturbed user's shared HIN list $\tilde{g^u}$. By this mechanism, the user's high-order patterns are maximum preserved since we select similar shared knowledge with high probability.

\noindent \textbf{User-item interaction perturbation}. After obtaining the perturbed $\tilde{g^u}$ (w.r.t. $\tilde{\mathcal{G}_s^u}$), we can extract a semantic guided item set $I^u_s$. The user-item interaction perturbation is conducted within the $I^u_s$ rather than the whole item set. Since our $\epsilon$-semantic guided interaction privacy is defined within the $I^u_s$, ignoring the items outside of $I^u_s$ has no effect on privacy guarantees. Besides, it also avoids introducing more irrelevant items and reduces the communication cost. In light of the user-related shared HIN having been perturbed in the first stage, we can directly apply RR to perturb $I^u_s$. However, in HIN-based recommendations, the size of $I^u_s$ is still large due to the relatively small number of shared HINs, thus introducing more irrelevant items. 

Inspired by \cite{DBLP:journals/corr/abs-2202-10209}, we propose a user-item interaction perturbation mechanism, which performs degree-preserving RR (DPRR) \cite{DBLP:journals/corr/abs-2202-10209} on each of semantic guided item set. Specifically, given user $u$ and $\tilde{\mathcal{G}^u_s}$, we can split semantic guided item set $I^u_s$ into $|\tilde{\mathcal{G}^u_s}|$ subsets. For each subset $I^u_{s_i}$, we use a adjacency list $a^u_{s_i}=(a^u_{s_i1},\ldots,a^u_{s_i|I^u_{s_i}|}) \in \{0,1\}^{|I^u_{s_i}|}$ to denote the user-item interactions. DPRR perturbs each bit of $a^u_{s_i}$ by first applying RR and then with probability $q^u_{s_i}$ to keep a result of 1 (a user-item interaction) unchanged. Thus the probability of each bit being perturbed to 1 is:

\begin{equation}
    \text{Pr}(\tilde{a}^u_{s_ij}=1) = \left\{ 
    \begin{aligned}
    &(1-p)q^u_{s_i}  && (\text{if} \:\: a^u_{s_ij} =1)\cr
    &pq^u_{s_i}  &&  (\text{if} \:\: a^u_{s_ij} =0).\cr
    \end{aligned}
\right.
\end{equation}
Assuming the true degree of user $u$ within the subset $I^u_{s_i}$ is $d^u_{s_i}$ (i.e., the number of 1 in $a^u_{s_i}$), according to the degree preservation property \cite{DBLP:journals/corr/abs-2202-10209}, the $q^u_{s_i}$ should be set as follows:
\begin{equation}
    q^u_{s_i} = \frac{d^u_{s_i}}{d^u_{s_i}(1-2p)+|I^u_{s_i}|p}.
\label{eq:q}
\end{equation}
In practice, the $q^u_{s_i}$ will be further clipped to $[0,1]$ to form probability. Note that the subset $I^u_{s_i}$ may not contain user-item interactions due to the perturbation on the shared HINs, in which case $q^u_{s_i}$=0. That is, we abandon a part of the interacted items, leading to semantic losses. Instead of that, we randomly select some items within $I^u_{s_i}$ so that the total degree is equal to the true degree $d^u$. We argue that in this way the semantics of user-item interactions are preserved in light of our shared HIN selection mechanism. 


\subsection{Heterogeneous Graph Neural Networks for Recommendation} 
Given a recovered meta-path, our HGNN first utilizes node-level attention to learn the weights of different neighbors under the meta-path. Then the weighted aggregated embeddings are fed into a semantic-level attention to aggregate embeddings under different meta-paths. Following this process, we give an illustration of obtaining user embeddings, and item embeddings are the same. 

\noindent \textbf{Node-level aggregation}. Let $h_{u_i}$ denotes the raw feature of a user $u_i$. Giving a meta-path $\rho_k$ and the recovered meta-path based neighbors $\mathcal{N}_{u_i}^{\rho_k}$ of $u_i$, the HGNN learns the weights of different neighbors via self-attention \cite{DBLP:conf/nips/VaswaniSPUJGKP17} followed by a softmax normalization layer:


\begin{equation}
    \alpha_{u_i u_j}^{\rho_k} = \text{softmax}_{u_j \in \mathcal{N}_{u_i}^{\rho_k}}(\sigma(\textbf{a}^T_{\rho_k} \cdot [\textbf{W}_{\rho_k} \cdot h_{u_i}||\textbf{W}_{\rho_k} \cdot h_{u_j}])),
\end{equation}
where $\textbf{W}_{\rho_k}$ and $\textbf{a}_{\rho_k}$ are the meta-path-specific learnable parameters. Note that $\mathcal{N}_{u_i}^{\rho_k}$ only keeps the user neighbors along with the meta-path. After obtaining the attention weights, the model performs node-level aggregations to get the meta-path based user embeddings:

\begin{equation}
   z^{\rho_k}_{u_i} = \sigma(\sum_{u_j \in \mathcal{N}_{u_i}^{\rho_k}} \alpha^{\rho_k}_{u_i u_j} \cdot h_{u_j}).
\end{equation}
Since the neighbors are all in the meta-path $\rho_k$, the semantics of $\rho_k$ are fused into the user's embeddings. Thus given the meta-path set $\mathcal{P} = \{\rho_1,\ldots,\rho_m \}$, we can obtain $m$ meta-path based embeddings $\{z^{\rho_1}_{u_i},\ldots, z^{\rho_m}_{u_i}\}$ of $u_i$.



\noindent \textbf{Semantic-level aggregation}. User embedding with the specific meta-path only contains a single semantic (e.g., U-M-U). After we obtain user embeddings from different meta-paths, an attention-based semantic-level aggregation is conducted to fuse different semantics. Specifically, The importance (attention weights) of specific meta-path $\rho_k$ is explained as averaging all corresponding transformed user embeddings, which is learned as follows:

\begin{equation}
\beta^{\rho_k} = \text{softmax}_{\rho_k \in \mathcal{P}}(\frac{1}{|\mathcal{U}|}\sum_{u_i \in \mathcal{U}}\textbf{q}^{\text{T}} \cdot \tanh(\textbf{W} \cdot z_{u_i}^{\rho_k}+\textbf{b})),
\end{equation}
where $\textbf{W}$ and $\textbf{q}$ are the semantic-level parameters that are shared for all meta-paths and $\textbf{b}$ is the bias vector. Then we perform semantic-level aggregations based on learned attention weights to obtain the final user embedding $z_{u_i}$: 

\begin{equation}
   z_{u_i} = \sum_{\rho_k \in \mathcal{P}} \beta^{\rho_k} \cdot z^{\rho_k}_{u_i}.
\end{equation}

\noindent \textbf{Ranking loss}. Through the above process, we can obtain the final individual user embedding $z_{u_i}$ and item embedding $z_{v_j}$ respectively. The ranking score is defined as the inner product of them: $\hat{y}_{u_i v_j} = z_{u_i}^\text{T} z_{v_j}$. Then, a typical Bayesian Personalized Ranking (BPR) loss function \cite{DBLP:conf/uai/RendleFGS09} is applied to optimize the parameters:

\begin{equation}
\mathcal{L}_{u_i} = -\sum_{v_j \in I^{u_i}}\sum_{v_k \notin I^{u_i} } \text{ln}\sigma(\hat{y}_{u_iv_j}-\hat{y}_{u_ju_k}),
\label{eq:loss}
\end{equation}
where $I^{u_i}$ is the positive items set and $v_k \notin I^{u_i}$ is the negative item which is uniformly sampled.

\subsection{Privacy Analysis}
In this section, we give an analysis of our proposed semantic-preserving user-item interactions publishing, which satisfies both $\epsilon_1$-semantic privacy and $\epsilon_2$-semantic guided interaction privacy.  

\begin{theorem}
    The semantic-preserving user-item interactions publishing mechanism achieves $\epsilon_1$-semantic privacy.
\end{theorem}

\begin{proof}
    Let $g^u$ and $\hat{g^u}$ denote any two user-related shared HIN lists which only differ in one bit, and any output $\tilde{g^u}$ after the first-stage perturbation (denoted as $\mathcal{M}^{skp}=\{\mathcal{M}^{skp}_1,\ldots,\mathcal{M}^{skp}_n\}$ w.r.t. $n$ selections). Assuming the total privacy budget is $\epsilon_1$ and each selection consumes $\frac{\epsilon_1}{n}$ privacy budget. Since each selection is independent, we have:
    \begin{equation*}
        \begin{aligned}
        \frac{\text{Pr}(\mathcal{M}^{skp}(g^u)=\tilde{g^u})}{\text{Pr}(\mathcal{M}^{skp}(\hat{g^u})=\tilde{g^u})}
        &= \frac{\Pi_{i=1}^n \text{Pr}(\mathcal{M}^{skp}_i(g^u, q, \mathcal{G}_s)=G_{s_i})}{\Pi_{i=1}^n \text{Pr}(\mathcal{M}^{skp}_i(\hat{g^u}, q, \mathcal{G}_s)=G_{s_i})}\\
        &=\Pi_{i=1}^n \frac{\text{Pr}(\mathcal{M}^{skp}_i(g^u, q, \mathcal{G}_s)=G_{s_i})}{\text{Pr}(\mathcal{M}^{skp}_i(\hat{g^u}, q, \mathcal{G}_s)=G_{s_i})} ,
        \end{aligned}
    \end{equation*}
    According to the EM, we have:
    \begin{equation*}
         \frac{\text{Pr}(\mathcal{M}^{skp}_i(g^u, q, \mathcal{G}_s)=G_{s_i})}{\text{Pr}(\mathcal{M}^{skp}_i(\hat{g^u}, q, \mathcal{G}_s)=G_{s_i})} \leq e^{\frac{\epsilon_1}{n}},
    \end{equation*}
    Thus
    \begin{equation*}
    \frac{\text{Pr}(\mathcal{M}^{skp}(g^u)=\tilde{g^u})}{\text{Pr}(\mathcal{M}^{skp}(\hat{g^u})=\tilde{g^u})} \leq \Pi_{i=1}^n e^{\frac{\epsilon_1}{n}}=e^{\epsilon_1}.
    \end{equation*}
    
\end{proof}

\begin{theorem}
The semantic-preserving user-item interactions publishing mechanism achieves $\epsilon_2$-semantic guided interaction privacy.
\end{theorem}

\begin{proof}
    After the first-stage perturbation, we can obtain the semantic guided item set $I^u_s$ and semantic guided adjacency list $a^u_s$ based on $\tilde{g^u}$. Let $\hat{a^u_s}$ denotes any adjacency list that only differs one bit with $a^u_s$. Without loss of generality, we assume $a^u_{s1} \neq \hat{a^u_{s1}}$. The second-stage perturbation is equivalent to first applying RR and then flipping each bit of 1 with probability $1-q^u_{s_i}$. Denoting the RR perturbation as $\mathcal{M}^{RR}$, we have:
    \begin{equation*}
    \begin{aligned}
    \frac{\text{Pr}(\mathcal{M}^{RR}(a^u_s)=\tilde{a^u_s})}{\text{Pr}(\mathcal{M}^{RR}(\hat{a^u_s})=\tilde{a^u_s})}
    &=\frac{\text{Pr}(a^u_{s1} \rightarrow \tilde{a^u_{s1}})\ldots \text{Pr}(a^u_{s|I^u_s|} \rightarrow \tilde{a^u_{s|I^u_s|}})}{\text{Pr}(\hat{a^u_{s1}} \rightarrow \tilde{a^u_{s}})\ldots \text{Pr}(\hat{a^u_{s|I^u_s|}} \rightarrow \tilde{a^u_{s|I^u_s|}})} \\
    &=\frac{\text{Pr}(a^u_{s1} \rightarrow \tilde{a^u_{s1}})}{\text{Pr}(\hat{a^u_{s1}} \rightarrow \tilde{a^u_{s1}})}\leq \frac{1-p}{p} \\
    &= e^{\epsilon_2}.
    \end{aligned}
    \end{equation*}

The subsequent flipping operation can be viewed as post-processing on the $\tilde{a^u_s}$, thus the whole perturbation also achieving $\epsilon_2$-semantic guided interaction privacy.
\end{proof}

\section{Experiments}
\subsection{Experimental Setup}

\textbf{Datasets}. We employ four real HIN datasets, including two academic datasets (ACM and DBLP) and two E-commerce datasets (Yelp and Douban Book), where the basic information is summarized in Table~\ref{tab:data}. The \textit{user} nodes and the \textit{private link types} are marked in bold. For ACM and DBLP, the \textit{item} nodes means authors.



\begin{table}[h]
  \centering
  \fontsize{8.5}{8.5}\selectfont
  \caption{Dataset statistics.}
    \begin{tabular}{c|c|c|c}
    \hline
    \textbf{Dataset} & \textbf{\# Nodes} & \textbf{\makecell*[c]{\# Private/Shared \\Links}} &\textbf{Meta-paths}
    \\
    \hline
    \hline
    ACM & \makecell*[c]{\textbf{Paper (P)}: 4025\\Author (A): 17431\\Conference (C): 14} & \makecell*[c]{\textbf{P-A}: 13407\\P-C: 4025\\} & \makecell*[c]{P-A-P \\P-C-P\\ A-P-A}  \\ \hline
    DBLP & \makecell*[c]{\textbf{Paper (P)}: 14328\\ Author (A): 4057\\Conference (C): 20} & \makecell*[c]{\textbf{P-A}: 19645\\ P-C: 14328\\} & \makecell*[c]{P-A-P\\P-C-P\\A-P-A} \\ \hline
    Yelp & \makecell*[c]{\textbf{User (U)}: 8743\\ Business (B): 3985\\Category (C): 511} & \makecell*[c]{\textbf{U-B}: 14896\\B-C: 11853\\} & \makecell*[c]{U-B-U\\U-B-C-B-U\\B-U-B}  \\ \hline
    \makecell*[c]{Douban \\Book} & \makecell*[c]{\textbf{User (U)}: 6793\\ Book (B): 8322\\Group (G): 2936\\ Author (A): 10801} & \makecell*[c]{\textbf{U-B}: 21179\\U-G: 664847 \\B-A: 8171} & \makecell*[c]{U-B-U\\U-G-U\\B-U-B\\B-A-B}  \\ \hline
    \end{tabular}%
\label{tab:data}%
\end{table}%

\begin{table*}[ht]
  \renewcommand\arraystretch{1.1}
  \centering
  \caption{Overall performance of different methods on Four datasets. The best result is in bold.}
  \resizebox{2.1\columnwidth}{!}{
    \begin{tabular}{ccccccccccccccc}
    \hline
    \multicolumn{2}{c||}{Model} & \multicolumn{1}{m{0.8cm}<{\centering}|}{HERec} & \multicolumn{1}{c|}{HAN} & \multicolumn{1}{c|}{NGCF} & \multicolumn{1}{m{1.2cm}<{\centering}|}{lightGCN} & \multicolumn{1}{c|}{RGCN} & \multicolumn{1}{m{0.8cm}<{\centering}}{HGT} & \multicolumn{1}{||c}{FedMF} & \multicolumn{1}{|m{1cm}<{\centering}|}{FedGNN} & \multicolumn{1}{m{0.8cm}<{\centering}|}{FedSog} & \multicolumn{1}{m{1.2cm}<{\centering}|}{PerFedRec}& \multicolumn{1}{m{1cm}<{\centering}|}{PFedRec}& \multicolumn{1}{m{1.4cm}<{\centering}}{SemiDFEGL} & \multicolumn{1}{||m{1.2cm}<{\centering}}{FedHGNN} \\
    \hline
    \hline
    \multirow{4}{*}{\begin{sideways}ACM\end{sideways}} & \multicolumn{1}{||l||}{H@5} & 0.3874 &\textbf{0.4152} &0.3845  & 0.3684 &    0.2929   & 0.3834 & \multicolumn{1}{||c}{0.0834} & 0.2608 & 0.2905 & 0.2516 & 0.2733 & 0.2065& \multicolumn{1}{||c}{\textbf{0.3593}} \\
\cline{2-15}          
& \multicolumn{1}{||l||}{H@10} & 0.4525 & 0.4727 & 0.4379& 0.4737 &   0.4619    & \textbf{0.5035} & \multicolumn{1}{||c}{0.1331} & 0.345 & 0.3642 & 0.3229 &0.3533&0.3083&  \multicolumn{1}{||c}{\textbf{0.4185}} \\
\cline{2-15}          
& \multicolumn{1}{||l||}{N@5} & 0.3222 & \textbf{0.335} &0.322& 0.2624 &   0.1752    & 0.2612 & \multicolumn{1}{||c}{0.056} & 0.193 & 0.2201 & 0.1824 &0.1982&0.1384& \multicolumn{1}{||c}{\textbf{0.2787}} \\
\cline{2-15}          
& \multicolumn{1}{||l||}{N@10} & 0.3333 &\textbf{0.3537}& 0.3393 & 0.2968 &    0.2302   &  0.3001& \multicolumn{1}{||c}{0.072} & 0.2202 & 0.2438 & 0.2055 &0.224&0.171& \multicolumn{1}{||c}{\textbf{0.298}} \\
    \hline
\multirow{4}{*}{\begin{sideways}DBLP\end{sideways}} & \multicolumn{1}{||l||}{H@5} & 0.3265 &\textbf{0.3877}&   0.3161    & 0.3256 &   0.387    &  0.3252& \multicolumn{1}{||c}{0.0998} & 0.2301 & 0.1978 & 0.1676 & 0.2843&0.1442&\multicolumn{1}{||c}{\textbf{0.3376}} \\
\cline{2-15}          
& \multicolumn{1}{||l||}{H@10} & 0.3882 &0.4498&  0.3895  & 0.4419 &    \textbf{0.5074}   & 0.4763 & \multicolumn{1}{||c}{0.1606} & 0.3252 & 0.2691 & 0.2619 &0.3863&0.2518& \multicolumn{1}{||c}{\textbf{0.4373}} \\
\cline{2-15}          
& \multicolumn{1}{||l||}{N@5} & 0.2586 &\textbf{0.33}&0.246    & 0.2281 &   0.2763    & 0.2264 & \multicolumn{1}{||c}{0.0603} & 0.167 & 0.14  & 0.105 &0.2005&0.091& \multicolumn{1}{||c}{\textbf{0.2481}} \\
\cline{2-15}          
& \multicolumn{1}{||l||}{N@10} & 0.2717 &\textbf{0.3503}&  0.27  & 0.2646 &   0.3151    & 0.2748 & \multicolumn{1}{||c}{0.0732} & 0.1963 & 0.163 & 0.1352 &0.2332&0.1253& \multicolumn{1}{||c}{\textbf{0.2778}} \\
\hline

\multirow{4}{*}{\begin{sideways}Yelp\end{sideways}} & \multicolumn{1}{||l||}{H@5} & 0.2322 &0.2877&  0.1831 & 0.2368 &    0.2844   & \textbf{0.3322} & \multicolumn{1}{||c}{0.0712} & 0.1801 & 0.1839 &0.1513 &0.1572&0.1903&\multicolumn{1}{||c}{\textbf{0.2178}} \\
\cline{2-15}          
& \multicolumn{1}{||l||}{H@10} & 0.3322 &  0.4077  &0.2958& 0.3684 &   0.3907    & \textbf{0.4635} & \multicolumn{1}{||c}{0.1259} & 0.2596 & 0.2715 & 0.237 &0.2249&0.28& \multicolumn{1}{||c}{\textbf{0.2977}} \\
\cline{2-15}          
& \multicolumn{1}{||l||}{N@5} & 0.1637 &  0.1929     &0.1127& 0.1881 &   0.2003    & \textbf{0.2311} &
\multicolumn{1}{||c}{0.0444} & 0.1221 & 0.1227 & 0.1002&0.1077&0.121 & \multicolumn{1}{||c}{\textbf{0.1578}} \\
\cline{2-15}          
& \multicolumn{1}{||l||}{N@10}& 0.1961&   0.2316  & 0.1493 & 0.2307   & 0.2346  &\textbf{0.2733} &
\multicolumn{1}{||c}{0.0619} & 0.1477 & 0.1508 & 0.1277 &0.1294&0.1497& \multicolumn{1}{||c}{\textbf{0.1834}} \\
\hline

\multirow{4}{*}{\begin{sideways}Db Book\end{sideways}} & \multicolumn{1}{||l||}{H@5} & 0.248 &0.2488&  0.1572 & 0.2927 &    0.3321   & \textbf{0.3431} & \multicolumn{1}{||c}{0.0859} & 0.1528 & 0.2505 &0.2039 &0.2842&0.0911&\multicolumn{1}{||c}{\textbf{0.3213}} \\
\cline{2-15}          
& \multicolumn{1}{||l||}{H@10} & 0.323 &  0.3602  &0.2331& 0.3902 &   0.4819    & \textbf{0.4973} & \multicolumn{1}{||c}{0.1411} & 0.22273 & 0.3464 & 0.2727 &0.3638&0.1367& \multicolumn{1}{||c}{\textbf{0.438}} \\
\cline{2-15}          
& \multicolumn{1}{||l||}{N@5} & 0.1767 &  0.1704     &0.1067& 0.2198 &   0.2239    & \textbf{0.2307} &
\multicolumn{1}{||c}{0.0529} & 0.1017 & 0.1743 & 0.1435&0.2037&0.0636 & \multicolumn{1}{||c}{\textbf{0.2135}} \\
\cline{2-15}          
& \multicolumn{1}{||l||}{N@10}& 0.2011&   0.2084  & 0.1289 & 0.2518   & 0.2671  &\textbf{0.2802} &
\multicolumn{1}{||c}{0.066} & 0.1257 & 0.2053 & 0.1658 &0.2295&0.0804& \multicolumn{1}{||c}{\textbf{0.2511}} \\
    \hline
\end{tabular}%
}
  \label{tab:overall}%
\end{table*}%

\textbf{Baselines}. Following \cite{wu2022federated}, we compare FedHGNN with two kinds of baselines: recommendation model based on centralized data-storage (including HERec \cite{DBLP:journals/tkde/ShiHZY19}, HAN \cite{DBLP:conf/www/WangJSWYCY19}, NGCF \cite{DBLP:conf/sigir/Wang0WFC19}, lightGCN \cite{DBLP:conf/sigir/0001DWLZ020}, RGCN \cite{DBLP:conf/esws/SchlichtkrullKB18}, HGT \cite{DBLP:conf/www/HuDWS20}) and federated setting for privacy-preserving (including FedMF \cite{DBLP:journals/expert/ChaiWCY21}, FedGNN \cite{wu2022federated}, FedSog \cite{DBLP:journals/tist/LiuYFPY22}, PerFedRec \cite{DBLP:conf/cikm/LuoXS22}, PFedRec\cite{DBLP:conf/ijcai/ZhangL0YZZY23}, SemiDFEGL\cite{DBLP:conf/www/QuTZNHSY23}). The details of them are shown in Appendix \ref{app:baselines}.

\textbf{Implementation Details}. For all the baselines, the node features are randomly initialized and the hidden dimension is set to 64. We tune other hyper-parameters to report the best performance. We keep the available heterogeneous information (e.g., meta-paths) the same for all HIN-based methods. We modify the loss function to be the BPR loss as the same with ours. In FedHGNN, the learning rate is set as 0.01, $\epsilon_1$ and $\epsilon_2$ are all set as 1. For each dataset, we first perform item clustering based on shared knowledge so that each item only belongs to one shared HIN. The number of shared HIN (number of clustering) is set as 20 for all datasets. The batch size (the number of participated clients in each round) is set as 32. For LDP and pseudo-interacted items, we set the hyper-parameters as the same with \cite{DBLP:journals/tist/LiuYFPY22}. Following \cite{DBLP:conf/cikm/LuoXS22}, we apply the leave-one-out strategy for evaluation and use HR@K and NDCG@K as metrics. We will also provide an implementation based on GammaGL \cite{DBLP:conf/sigir/LiuYZHZWZHWS23}.

\subsection{Overall Performance}
Table \ref{tab:overall} shows the overall results of all baselines on four datasets. The following findings entail from Table \ref{tab:overall}: (1) FedHGNN outperforms all the FedRec models by a big margin (up to 34\% in HR@10 and 42\% in NDCG@10), which demonstrates the effectiveness of our model.  Surprisingly, FedHGNN also outperforms several centralized models (notably non-HIN based methods, e.g., NGCF), indicating the significance of utilizing rich semantics of HIN in FedRec. We also assume the perturbation can be seen as an effective data augmentation to alleviate cold-start issues. Since we find the interactions of some inactive users slightly increase after perturbation.  (2) Among centralized baselines, HIN-based methods perform better, especially on sparse datasets (e.g., DBLP), owing to introducing additional semantic information to alleviate cold-start issues. It has also been observed that GNN-based methods (HAN, RGCN, and HGT) achieve better results than non-GNN based methods (HERec), indicating that GNNs are more potent in capturing semantic information. (3) Among federated baselines, FedMF performs poorly because it ignores the high-order interactions which are significant for cold-start recommendation. Other federated models (FedGNN, FedSog, PerFedRec, and SemiDFEGL) improve this by privacy-preserving graph expansion (FedSog assumes social relation is public). SemiDFEGL performs relatively poorly since it focuses more on parameter efficiency and reducing communication costs rather than utility. It's surprising to find that PFedRec outperforms FedMF by a large margin and even outperforms some GNN-based FedRec baselines. We discover that it’s attributed to the personal item embeddings and the same parameter initialization of score functions, which is an interesting finding and worth studying in future works. Compared to these methods, our FedHGNN further considers semantic information with theoretically guaranteed privacy protection.

\begin{table*}[ht]
  \renewcommand\arraystretch{1.1}
  \centering
  \caption{Performance of different variants of FedHGNN on three datasets.}
    \begin{tabular}{m{0.7cm}<{\centering}lccccccccc}
    \hline
    \multicolumn{2}{c||}{Model} & \multicolumn{1}{c|}{$\text{FedHGNN}^{*}$} & \multicolumn{1}{c|}{+RR} & \multicolumn{1}{c|}{+DPRR} & \multicolumn{1}{c|}{+SRR} &\multicolumn{1}{c|}{+SDPRR} &\multicolumn{1}{c|}{+E}&\multicolumn{1}{c|}{+ESRR}&\multicolumn{1}{c}{+$\text{ESDPRR}^{*}$}& \multicolumn{1}{||m{1.2cm}<{\centering}}{FedHGNN}  \\
    \hline
    \hline
    \multirow{4}{*}{ACM} & \multicolumn{1}{||l||}{HR@5} & 0.3118 &0.0495 &0.1749  & 0.3437 &    \textbf{0.389}&0.3461&0.2959 &0.3475 &\multicolumn{1}{||c}{0.3593}\\
\cline{2-11}          
& \multicolumn{1}{||l||}{HR@10} & 0.3961 & 0.1118 & 0.2268& 0.4998 &   \textbf{0.5027}&0.4004&0.3861 &0.4069 &\multicolumn{1}{||c}{0.4185}     \\
\cline{2-11}          
& \multicolumn{1}{||l||}{NDCG@5} & 0.2293 & 0.0345 &0.1326& 0.2312 &   \textbf{0.2865}&0.266&0.215 &0.266 &\multicolumn{1}{||c}{0.2787}     \\
\cline{2-11}          
& \multicolumn{1}{||l||}{NDCG@10} & 0.2567 &0.0491& 0.1492 & 0.2845 &    \textbf{0.323}&0.2835&0.2438 &0.2852 &\multicolumn{1}{||c}{0.298}    \\
    \hline
\multirow{4}{*}{DBLP} & \multicolumn{1}{||l||}{HR@5} & 0.2824 &0.0694&   0.1678    & 0.2224 &   0.3346&0.2616&0.2729 &0.3156 &\multicolumn{1}{||c}{\textbf{0.3376}}   \\
\cline{2-11}          
& \multicolumn{1}{||l||}{HR@10} & 0.3934 &0.1237&  0.2394  & 0.3154 &    \textbf{0.4557}&0.3701&0.3718 &0.4227 &\multicolumn{1}{||c}{0.4373}    \\
\cline{2-11}          
& \multicolumn{1}{||l||}{NDCG@5} & 0.2176 &0.0429&0.1115    & 0.1458 &   \textbf{0.2484}&0.1835&0.1929 &0.2273 &\multicolumn{1}{||c}{0.2481}     \\
\cline{2-11}          
& \multicolumn{1}{||l||}{NDCG@10} & 0.241 &0.0602&  0.1413  & 0.1757 &   \textbf{0.2801}&0.2176&0.2249 &0.2619 &\multicolumn{1}{||c}{0.2778}    \\
\hline
\multirow{4}{*}{Yelp} & \multicolumn{1}{||l||}{HR@5} & \textbf{0.2583} &0.0663&  0.1383 & 0.1172 &    0.2364&0.2244&0.223 &0.1871 &\multicolumn{1}{||c}{0.2178}   \\
\cline{2-11}          
& \multicolumn{1}{||l||}{HR@10} & \textbf{0.3482} &  0.1232  &0.2079& 0.1803 &   0.3245&0.3242&0.3257 &0.2624 &\multicolumn{1}{||c}{0.2977}     \\
\cline{2-11}          
& \multicolumn{1}{||l||}{NDCG@5} & \textbf{0.1859} &  0.0392     &0.0963& 0.0672 &   0.171&0.152&0.1538 &0.1321 &\multicolumn{1}{||c}{0.1578}   \\
\cline{2-11}          
& \multicolumn{1}{||l||}{NDCG@10}& \textbf{0.2201}&   0.0575  & 0.1185 & 0.0789   & 0.1976&0.1875&0.1804&0.1563 &\multicolumn{1}{||c}{0.1834}  \\
    \hline
    \end{tabular}%
  \label{tab:ablation}%
\end{table*}%

\subsection{Ablation Study}
\label{ab_study}

To have an in-depth analysis of our two-stage perturbation mechanism, we conduct ablation studies to dissect the effectiveness of different modules. We design 7 variants based on FedHGNN and the performance of these variants is outlined in Table \ref{tab:ablation}. $\textit{FedHGNN}^{*}$ is the FedHGNN model without two-stage perturbation. \textit{RR} means random response and $\textit{DPRR}$ is the degree-preserving RR. \textit{+S} indicates adding corresponding perturbation to each semantic guided adjacency list $a^u_{s_i}$, otherwise to each user's adjacency list $a^u$. Note that $\textit{SDPRR}^{*}$ indicates performing $\textit{DPRR}$ to the whole semantic guided adjacency list $a^u_s$, which is the only difference with our FedHGNN. \textit{+E} indicates adding EM perturbation. We set $\epsilon_2=\epsilon_2=1$ for all variants except that $\epsilon_2=6$ for RR-related variants, due to a smaller $\epsilon_2$ makes the graph denser, which sharply increases training time and consumed memory.

From the table, we have several findings: (1) The performance of FedHGNN is even superior to the model without perturbation in ACM and DBLP, and removing the first-stage perturbation ($\textit{+SDPRR}$) can achieve better results. In contrast, $\textit{+DPRR}$ obtains worse results, indicating that the main factor for improving the FedHGNN is the $\textit{DPRR}$ on each semantic guided adjacency list. Considering the relatively sparse datasets and a slight increment of some inactive user’s interactions after the $\textit{SDPRR}$, we conclude that $\textit{SDPRR}$ has the ability to tackle data sparsity in recommendations since it augments data in a semantic-preserving manner meanwhile keeping data diversity. (2) Pure $\textit{RR}$ and $\textit{DPRR}$ perform poorly since they perturb user-item interactions randomly without considering semantic preserving. Pure $\textit{RR}$ performs even worse due to it making a graph denser and causing perturbation enlargements \cite{DBLP:conf/aaai/ZhangWZSZZ22}. $\textit{DPRR}$ preserves degrees but fails to preserve user-item interaction patterns. Thus we can draw a conclusion that semantic-preserving requires both degree-preserving and feature-preserving. Perturbation within the semantic guided item set (\textit{+SRR} and \textit{+SDPRR}) performs much better, which further verifies our conclusion. (3) Adding first-stage perturbation will harm the performance but is necessary, otherwise we can't protect the user's high-order patterns. Thanks to our designed similarity-based EM, the user's high-order patterns are maximum preserved and the performance has not decreased dramatically. Note that FedHGNN also outperforms $\textit{+ESDPRR}^{*}$, indicating we should keep the diversity of user-item interactions after EM, i.e., the interacted items should exist in each selected shared HIN.

\begin{figure}[!t]
    \centering
    \includegraphics[scale=.1]{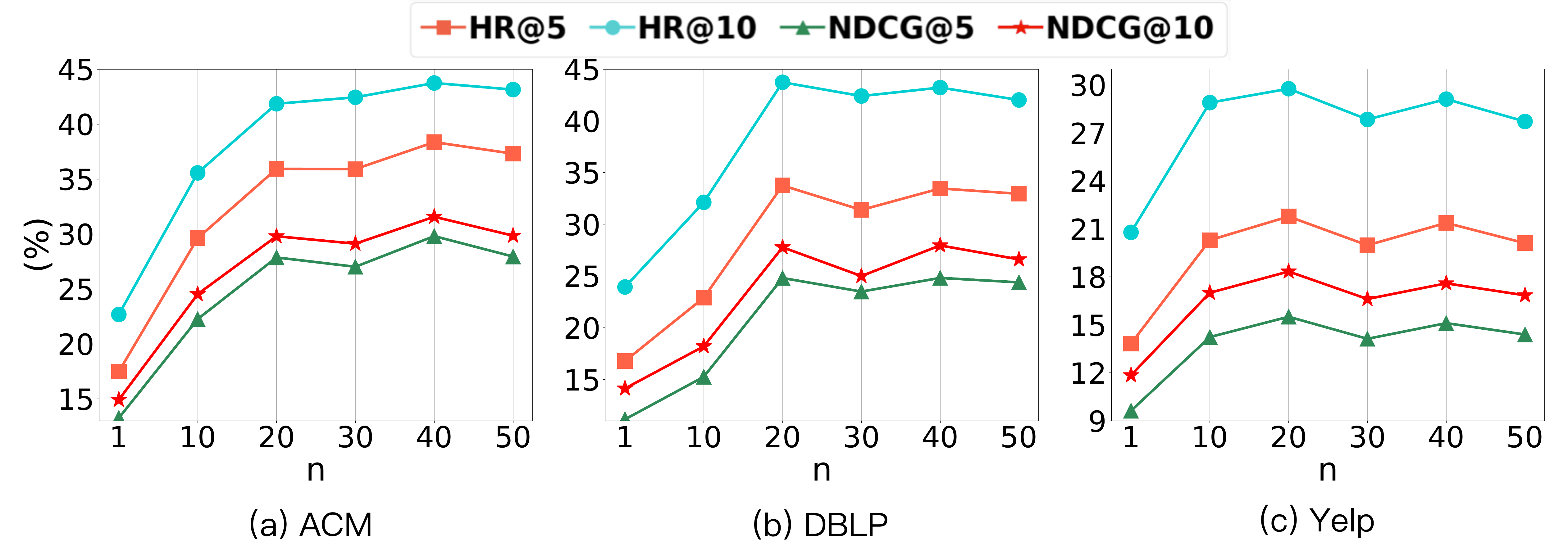}
    \caption{Effects of different number $n$ of shared HINs}
    \label{fig:shared_num}
\end{figure}

\begin{figure}[t]
    \centering
    \includegraphics[scale=.1]{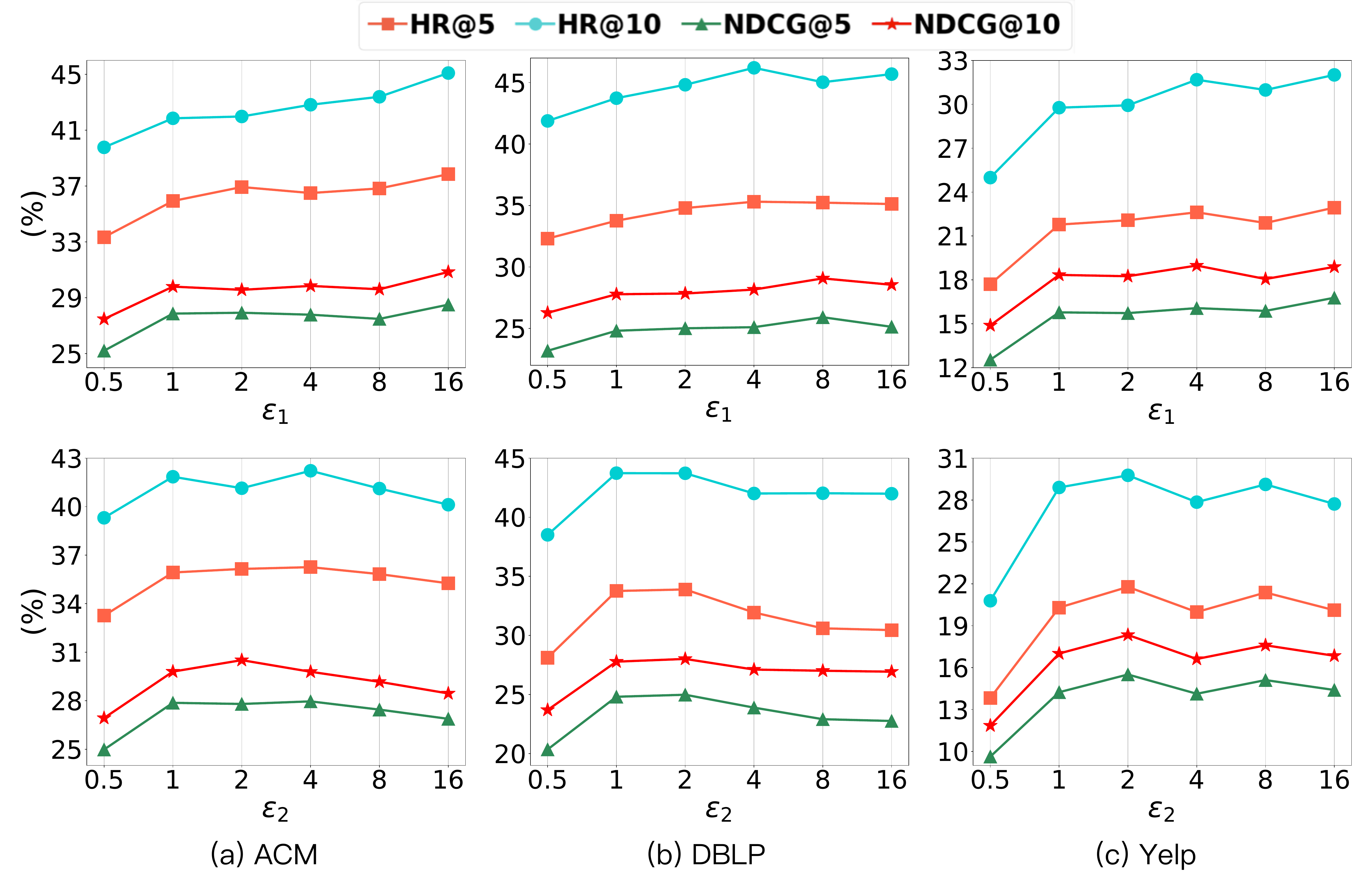}
    \caption{Effects of different privacy budget $\epsilon_1$ and $\epsilon_2$}
    \label{fig:eps}
\end{figure}

\subsection{Parameter Analysis}
In this section, we investigate the impacts of some significant parameters in FedHGNN, including the number of shared HINs, as well as the privacy budgets $\epsilon_1$ and $\epsilon_2$. 

\textbf{Analysis of different numbers of shared HINs}. To demonstrate the effects of different numbers $n$ of shared HINs, we fix other hyper-parameters unchanged and vary $n$ to compare the performance. The results are depicted in Figure \ref{fig:shared_num}. Considering two extreme conditions: when $n=1$, the two-stage perturbation degenerates to solely second-stage perturbation, i.e., perturbation by DPRR on the whole item set, which fails to preserve user-item interaction patterns as discussed in Section \ref{ab_study}; When $n=|I|$, according to Eq. (\ref{eq:q}), it is equivalent to performing RR in each 1's bit after the first-stage perturbation, which intuitively perform better than $n \leq |I|$. According to this theory, the performance will increase when $n$ is larger. However, as can be seen, the performance of all datasets has a dramatic incremental trend at the initial stage of increasing $n$ ($n \leq 20$), then the curve becomes smooth and even has a decrement trend ($n \geq 20$). We attribute this phenomenon to that the model with perturbed user-item interactions is performing better than with true interactions, as depicted in Table \ref{tab:ablation}, and a large $n$ will reduce this effect. In summary, $n$ controls the trade-off between utility and privacy, a larger $n$ may bring relatively higher utility but weaker privacy protection, since the attacker can conclude the user-item interactions within a smaller scope. 

\textbf{Analysis of different privacy budgets}. To analyze the effects of different $\epsilon_1$ and $\epsilon_2$, we fix one parameter as 1 and change another one from 0.5 to 16 to depict the model performance in Figure. \ref{fig:eps}. $\epsilon_1$ controls the protection strength of user's high-order patterns (related-shared HINs). We can see that the metrics gradually increase with $\epsilon_1$, indicating the user's high-order patterns are significant for recommendation and these patterns are undermined when $\epsilon_1$ is too small (e.g., 0.5). When fixing $\epsilon_1=1$, the performance curve of $\epsilon_2$ will first increase and then slightly decrease. We suppose that due to the perturbation of user's high-order patterns in the first stage, the second stage perturbation is conducted on contaminated interaction patterns. As a result, the performance may still drop when $\epsilon_2$ is large. It also shows that conducting moderate perturbation will promote the performance (e.g., $\epsilon_2=1$). 




\section{Conclusion}
In this paper, we first explore the challenge problem of HIN-based FedRec. We formulate the privacy in federated HIN and propose a semantic-preserving user-item publishing method with rigorous privacy guarantees. Incorporating this publishing method into advanced heterogeneous graph neural networks, we propose a FedHGNN framework for recommendation. Experiments show that the model achieves satisfied utility under an accepted privacy budget. 

\begin{acks}
This work is supported in part by the National Natural Science Foundation of China (No. U20B2045, 62192784, U22B2038, 62002029, 62172052), the JSPS KAKENHI (No. JP22H00521, JP22H03595), and the BUPT Excellent Ph.D. Students Foundation (No. CX2021118).
\end{acks}

\bibliographystyle{ACM-Reference-Format}
\bibliography{sample-base}

\newpage
\appendix

\section{Related work}
\label{app:rel}
\textbf{HIN-based recommendation}. HIN contains rich semantics for recommendation, which has been extensively studied in recent years \cite{DBLP:conf/kdd/ZhengMGZ21, DBLP:conf/cikm/XuLHLX019, DBLP:conf/cikm/HuSZY18}. Specifically, HERec \cite{DBLP:journals/tkde/ShiHZY19} utilizes a meta-path based random walk to generate node sequences and designs multiple fusion functions to enhance the recommendation performance. MCRec \cite{DBLP:conf/kdd/HuSZY18} designs a co-attention mechanism to explicitly learn the interactions between users, items, and meta-path based context. In recent years, HGNNs have been introduced for HIN modeling. To tackle the heterogeneous information, one line aggregates neighbors after transforming heterogeneous attributes into the same embedding space \cite{DBLP:conf/www/HuDWS20, DBLP:conf/esws/SchlichtkrullKB18}. Typically, RGCN \cite{DBLP:conf/esws/SchlichtkrullKB18} aggregates neighbors for each relation type individually. HetGNN \cite{DBLP:conf/kdd/ZhangSHSC19} adopts different RNNs to aggregate nodes with different types. HGT \cite{DBLP:conf/www/HuDWS20} introduces transformer architecture \cite{DBLP:conf/nips/VaswaniSPUJGKP17} for modeling heterogeneous node and edge types. Another line is performing meta-path based neighbor aggregation \cite{DBLP:conf/www/WangJSWYCY19, DBLP:conf/kdd/FanZHSHML19, DBLP:conf/aaai/JiZWSWTLH21}. HAN \cite{DBLP:conf/www/WangJSWYCY19} proposes a dual attention mechanism to learn the importance of different meta-paths. HGSRec \cite{DBLP:conf/aaai/JiZWSWTLH21} further designs tripartite heterogeneous GNNs to perform shared recommendations. Unlike HAN performing homogeneous neighbor aggregation, Meirec \cite{DBLP:conf/kdd/FanZHSHML19} proposes a meta-path guided heterogeneous neighbor aggregation method for intent recommendation. Despite the great effectiveness of these HIN-based recommendations, they are all designed under centralized data storage and not geared for the federated setting with privacy-preserving requirements. 

\textbf{Federated recommendation}.
Federated learning (FL) is proposed to collaboratively train a global model based on the distributed data\cite{DBLP:conf/aistats/McMahanMRHA17, DBLP:series/lncs/12500,DBLP:journals/corr/abs-2104-07145}. Accordingly, the global model in FedRec is collectively trained based on the user's local interaction data \cite{DBLP:conf/aaai/LiangP021, DBLP:conf/kdd/MuhammadWOTSHGL20, DBLP:journals/tois/LinPYM23}. Each client maintains a local recommendation model and uploads intermediate data to the server for aggregation. In this process, the user’s interaction behaviors (the set of interacted items or rating scores) should be protected \cite{DBLP:conf/www/0001WSLZW22, DBLP:conf/www/YuanYNCHY23, DBLP:journals/corr/abs-2205-11857}. FCF \cite{DBLP:journals/corr/abs-1901-09888} is the first FedRec framework, which is based on the traditional collaborative filter (FCF). The user embeddings are stored and updated locally while the gradients of item embeddings are uploaded to the server for aggregation. FedMF \cite{DBLP:journals/expert/ChaiWCY21} proves that the uploaded gradients of two continuous steps can also leak user privacy and thus applies homomorphic encryption to encrypt gradients. SharedMF \cite{DBLP:journals/corr/abs-2008-07759} utilizes secret sharing instead of homomorphic encryption for better efficiency and FR-FMSS \cite{DBLP:conf/recsys/LinP021} further randomly samples fake ratings for better privacy. To achieve personalization, PFedRec removes user embedding and adds a parameterized score function. It keeps the score function locally updated and finetunes the global item embeddings. Recently, GNN-based FedRec has emerged \cite{wu2022federated, DBLP:journals/tist/LiuYFPY22, DBLP:conf/cikm/LuoXS22}. FedGNN \cite{wu2022federated} applies local differential privacy (LDP) to upload gradient and samples pseudo-interacted items for anonymity. Besides, a trusted third-party server is utilized to obtain high-order neighbors. To perform federated social recommendations, FedSoG \cite{DBLP:journals/tist/LiuYFPY22} employs a relation attention mechanism to learn local node embeddings and proposes a pseudo-labeling method to protect local private interactions. Considering personalization and communication costs, PerFedRec \cite{DBLP:conf/cikm/LuoXS22} clusters users and learns a personalized model by combining different levels of parameters. To reduce communication costs, SemiDFEGL \cite{DBLP:conf/www/QuTZNHSY23} proposes a semi-decentralized federated ego graph learning framework, which incorporates device-to-device communications by a fake common items sharing mechanism. Neither these FedRec methods utilize the rich semantics of HINs nor have rigorous privacy guarantees.

\section{Description of Baselines}
\label{app:baselines}

The detailed descriptions of baselines are presented as follows:

\begin{itemize}[leftmargin=*]

\item \textbf{HERec} \cite{DBLP:journals/tkde/ShiHZY19} is a HIN-enhanced recommender framework based on matrix factorization. It first trains HIN-based embeddings and then incorporates them into matrix factorization.

\item \textbf{HAN} \cite{DBLP:conf/www/WangJSWYCY19} introduces meta-path based neighbor aggregation to learn node embeddings. We utilize the learned embeddings to perform recommendations in an end-to-end fashion.

\item \textbf{NGCF} \cite{DBLP:conf/sigir/Wang0WFC19} is a GNN-based recommender method which learns embedding via message passing on user-item bipartite graph.

\item \textbf{lightGCN} \cite{DBLP:conf/sigir/0001DWLZ020} improves NGCF by removing feature transformation and nonlinear activation.

\item \textbf{RGCN} \cite{DBLP:conf/esws/SchlichtkrullKB18} utilizes GCN to learn node embeddings of knowledge graph. It performs message-passing along different relations.

\item \textbf{HGT} \cite{DBLP:conf/www/HuDWS20} proposes a transformer architecture for HIN modeling, which conducts heterogeneous attention when aggregating neighbors.

\item \textbf{FedMF} \cite{DBLP:journals/expert/ChaiWCY21} is a FedRec framework based on matrix factorization. The user embedding is updated locally and the encrypted item gradient is uploaded to the server for aggregation.
 
\item \textbf{FedGNN} \cite{wu2022federated} is a GNN-based FedRec framework. It obtains high-order user neighbors through a third-party trustworthy server and utilizes pseudo-interacted item sampling to achieve privacy-preserving. 

\item \textbf{FedSog} \cite{DBLP:journals/tist/LiuYFPY22} is another GNN-based FedRec framework. It proposes a relational graph attention network to perform social recommendations. 

\item \textbf{PerFedRec} \cite{DBLP:conf/cikm/LuoXS22} is a personalized FedRec method. It performs user clustering on the server and then aggregates parameters within each cluster to achieve personalization.

\item \textbf{PFedRec} \cite{DBLP:conf/ijcai/ZhangL0YZZY23} is another personalized FedRec method. Compared to FedMF, it removes user embeddings and adds a parameterized score function. It achieves personalization by locally updating score function and finetuning item embeddings.

\item \textbf{SemiDFEGL} \cite{DBLP:conf/www/QuTZNHSY23} is a semi-decentralized GNN-based FedRec framework. It introduces device-to-device collaborative learning to reduce communication costs and utilizes fake common items to exploit high-order graph structures.

\end{itemize}

\section{Algorithm}
\label{app:alg}
The whole process of FedHGNN is presented in Algorithm \ref{alg}. Each client executes the function \texttt{Semantic\_preserving\_perturb} before federated training to obtain the perturbed local user-item interaction data $\tilde{a^u_s}$ and upload it to the server. Server obtains the required meta-path based neighbors $\{\mathcal{N}_u^{\rho}\}_{\rho \in \mathcal{P}, u \in U}$ for each client based on $\{\tilde{a^u_s}\}_{u \in U}$ and then distributes $\{\mathcal{N}_u^{\rho}\}_{\rho \in \mathcal{P}, u \in U}$ to each client. After this one-step distribution, clients begin to collaboratively train a global recommendation model based on recovered neighbors. In each communication round, the sampled clients locally execute the function $\texttt{HGNN\_for\_rec}$ to train the recommendation model and upload the gradients to the server for aggregation. Note that for better privacy protection, we also sample pseudo items during local training and apply LDP to the uploaded gradients. This communication is executed multiple times until the model convergence.

\begin{algorithm}
    \renewcommand{\algorithmicrequire}{\textbf{Input:}}
    \renewcommand{\algorithmicensure}{\textbf{Output:}}
    \caption{FedHGNN} 
    \label{alg}
    \begin{algorithmic}[1]
        \REQUIRE  Meta-paths:$\mathcal{P}$; Learning rate:$\eta$; Shared HIN set: $\mathcal{G}_s$; The number of pseudo items:$p$
        \ENSURE Model parameters and embeddings: $\Theta$
        \STATE Server initializes $\Theta$ and distributes $\mathcal{G}_s$ to each client $u \in U$
        \FOR{each client $u \in U$}
            \STATE $\tilde{a^u_s}$ = $\texttt{Semantic\_preserving\_perturb}$($u$, $\mathcal{G}_s$)
            \STATE Upload $\tilde{a^u_s}$ to server
        \ENDFOR

        \STATE Server obtains meta-path based neighbors $\{\mathcal{N}_u^{\rho}\}_{\rho \in \mathcal{P}, u \in U}$ based on $\{\tilde{a^u_s}\}_{u \in U}$ 
        \STATE Distributes $\{\mathcal{N}_u^{\rho}\}_{\rho \in \mathcal{P}, u \in U}$ to corresponding client $u$
        \WHILE{\textit{not converge}}
            \STATE Server samples a client set $U_c \subset U$
            \STATE Distributes $\Theta$ to each client $u \in U_c$
            \FOR{each client $u \in U_c$}
                \STATE $\tilde{\text{g}_{\Theta}^u}= \texttt{HGNN\_for\_rec}$($\{N_u^{\rho}\}_{\rho \in \mathcal{P}}, \mathcal{P}, \Theta$)
                \STATE Upload $\tilde{\text{g}_{\Theta}^u}$ to server
            \ENDFOR
            \STATE Server aggregates $\{\tilde{\text{g}_{\Theta}^u}\}_{u \in U_c}$ into $\bar{\text{g}_{\Theta}}$
            \STATE Update $\Theta = \Theta - \eta \cdot \bar{\text{g}_{\Theta}}$ 
        \ENDWHILE 
        
        ~\\
        \STATE $\textbf{Function}$ 
        $\texttt{Semantic\_preserving\_perturb}$ ($u$, $\mathcal{G}_s$):
        \\// User-related shared HIN perturbation
        \FOR{$s = 1, \cdots, |\mathcal{G}_s^u|$} 
            \STATE Select a shared HIN $G_s$ from $\mathcal{G}_s$ with probability by Eq. (3)
        \ENDFOR
        \STATE Obtain perturbed $\tilde{g^u}$ and $\tilde{\mathcal{G}_s^u}$ based on above selections
        \\// User-item interaction perturbation
        \STATE Obtain user-item adjacency list $a^u_s$ based on $\tilde{g^u}$ and $\tilde{\mathcal{G}_s^u}$
        \STATE Obtain $\tilde{a^u_s}$ by perturbing each bit of $a^u_s$ with probability by Eq. (4)
        \STATE \textbf{return}  $\tilde{a_s^u}$
    
        ~\\
        \STATE $\textbf{Function}$ $\texttt{HGNN\_for\_rec}$ ($\{N_{u_i}^{\rho}\}_{\rho \in \mathcal{P}}, \mathcal{P}, \Theta$):
        \STATE Initialize local parameters with $\Theta$
        \\//user embedding
        \FOR{each $\rho \in \mathcal{P}$} 
            \STATE Obtain meta-path based user embeddings $\{z^{\rho}_{u_i}\}_{u_i \in U}$ using Eq. (7)
            \STATE Calculate $\beta^{\rho} $ using Eq. (8)
        \ENDFOR
        \STATE Obtain final user embedding $z_{u_i}$ using Eq. (9)
        \\//item embeddings $\{z_{v_i}\}_{v_i \in V}$ are obtained in the same way
        \STATE sample $p$ pseudo items as positive items
        \STATE Calculate BPR loss $\mathcal{L}_{u_i}$ with $p$ using Eq. (10)
        \STATE Calculate gradients $\text{g}_{\Theta}$ by $\frac{\partial \mathcal{L}_{u_i}}{\partial \Theta}$
        \STATE $\tilde{\text{g}_{\Theta}} \leftarrow$ LDP($\text{g}_{\Theta}$)
        \STATE \textbf{return} $\tilde{\text{g}_{\Theta}}$
            
    
    \end{algorithmic}
\end{algorithm}

\section{Additional results of the perturbation}
\label{app:perturbation}

Since we only perturb edges (user-item interactions), the number of nodes is unchanged. We calculate the number and proportion of edges that are unchanged after perturbation, which are presented in Table \ref{tab:perturbation}. We can see that the proportion of edges unchanged after perturbation is extremely small (less than 1\%) in all datasets, indicating that the privacy is maximum protected. In fact, the perturbation algorithm is based on differential privacy (DP) as discussed before, thus the adversary is hard to distinguish whether the unchanged edges really exist in the original graph, which already provides rigorous privacy guarantees. As for the utility, we conduct perturbation in a semantic-preserving manner, thus the semantics are maximum preserved and the number of unchanged edges can’t reflect whether the semantics are lost. Besides, experiments also verify the effectiveness of our perturbation algorithm.

\begin{table}[h]
  \centering
  \fontsize{8.5}{8.5}\selectfont
  \caption{The number and proportion of edges that are unchanged after perturbation on four datasets.}
    \begin{tabular}{c|c|c|c}
    \hline
    \textbf{Dataset} & \textbf{\# Edges} & \textbf{\makecell*[c]{\# Edges unchanged}} &\textbf{Proportion}
    \\
    \hline
    \hline
    ACM & 9703 & 19 & 0.002  \\ \hline
    DBLP & 15368 & 79 & 0.005 \\ \hline
    Yelp & 11187 & 11 & 0.001  \\ \hline
    Douban Book & 17083 & 150 & 0.009 \\ \hline
    \end{tabular}%
\label{tab:perturbation}%
\end{table}%

\section{convergence analysis}
\label{app:convergence}

Following \cite{DBLP:conf/icml/BaekJJYH23}, we give convergence analysis in Figure \ref{fig:conv}, which shows the convergence plots of different FedRec models in ACM and Douban Book datasets. We visualize HR@10 on different communication rounds. It can be observed that our FedHGNN converges rapidly with better results, and we conjecture that by utilizing the recovered semantics, FedHGNN can rapidly capture the user-item interaction patterns. Also, the communication cost of each round of FedHGNN is moderate compared to baselines. Therefore, although some baselines achieve faster convergence (e.g., FedMF and SemiDFEGL), considering the relatively large improvements in performance, it's acceptable of the trade-off between communication cost and utility in FedHGNN.

\begin{figure}[!t]
    \centering
    \includegraphics[scale=.165]{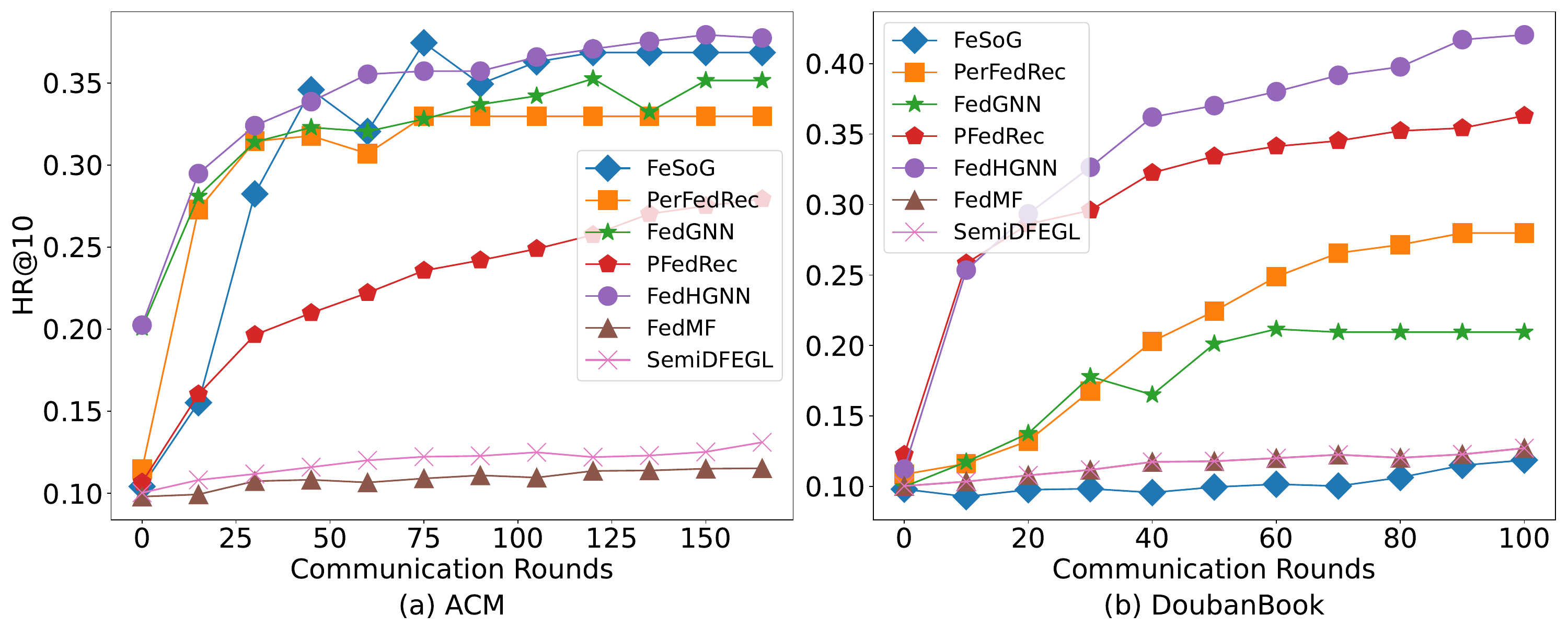}
    \caption{Coonvergence plots of different FedRecs in ACM and Douban Book datasets}
    \label{fig:conv}
\end{figure}

\end{document}